\documentclass[11pt]{article}
\usepackage[margin=1in]{geometry}

\usepackage{algorithm}
\usepackage{algorithmic}

\usepackage[utf8]{inputenc} 
\usepackage[T1]{fontenc}    
\usepackage{booktabs}       
\usepackage{amsfonts}       
\usepackage{nicefrac}       
\usepackage{microtype}      
  
\newcommand{\E}{\mathbb{E}}

\usepackage{mathtools}

\usepackage{subfigure}
\usepackage{amsthm,amsmath,amssymb,graphicx}%
\usepackage{enumitem}
\setlist{nosep,after=\vspace{\baselineskip}}

\usepackage{authblk}

\DeclareMathOperator*{\argmin}{arg\,min}

\newtheorem{definition}{Definition}
\newtheorem{lemma}{Lemma}
\newtheorem{corollary}{Corollary}
\newtheorem{proposition}{Proposition}

\newtheorem{theorem}{Theorem}

\def\conv{\rm{conv}}
\newcommand{\mat}[1]{\mathbf{#1}}

\begin{document} 
\date{\today}
\title{Cost-Optimal Learning of Causal Graphs}
\author[1,*]{Murat Kocaoglu}
\author[1,\textdagger]{Alexandros G. Dimakis}
\author[1,\textdaggerdbl]{Sriram Vishwanath}
\affil[1]{\small Department of Electrical and Computer Engineering, The University of Texas at Austin, USA}
\affil[ ]{\small \textit \textsuperscript{*} mkocaoglu@utexas.edu \textsuperscript{\textdagger}dimakis@austin.utexas.edu  \textsuperscript{\textdaggerdbl}sriram@ece.utexas.edu}
\renewcommand\Authands{ and }

\maketitle

\begin{abstract}
We consider the problem of learning a causal graph over a set of variables with interventions. We study the cost-optimal causal graph learning problem: For a given skeleton (undirected version of the causal graph), design the set of interventions with minimum total cost, that can uniquely identify any causal graph with the given skeleton. We show that this problem is solvable in polynomial time. Later, we consider the case when the number of interventions is limited. For this case, we provide polynomial time algorithms when the skeleton is a tree or a clique tree. For a general chordal skeleton, we develop an efficient greedy algorithm, which can be improved when the causal graph skeleton is an interval graph. 
\end{abstract}

\section{Introduction}
Causal inference is important for many applications including, among others, biology, econometrics and medicine~\cite{Chalupka2017,Wentrup2016, Ramsey2010}.
Randomized trials are the golden standard for causal inference since they lead to reliable conclusions with minimal assumptions. 
The problem is that enforcing randomization to different variables in a causal inference problem can have significant and varying costs. 
A causal discovery algorithm should take these costs into account and optimize experiments accordingly. 

In this paper we formulate this problem of learning a causal graph when there is a cost for intervening on each variable.
We follow the structural equation modeling framework~\cite{Pearl2009,Spirtes2001} and use interventions, i.e., experiments. To perform each intervention, a scienstist randomizes a set of variables and collects new data from the perturbed system. 
For example, suppose the scientist wants to observe the causal effect of the variable \textit{smoking} on the variable \textit{cancer}. Suppose she decides to perform an intervention on the \textit{smoking} variable. This entails forcing a random subset of participants to smoke, irrespective of them being smokers or non-smokers. This intervention would be (clearly) hard to perform, so the cost of intervening on the variable \textit{smoking}
should be set pretty high. An intervention on the second variable \textit{cancer} would be physically impossible, since there is no mechanism for the scientist to enforce this variable to take a value for a participant. We should therefore, be setting the cost of intervening on the \textit{cancer} variable to infinity. 

In this paper we study the following problem: We want to learn a causal graph and for each variable we are given a cost. For each intervention set, the cost is the sum of the costs of all the variables in the set. 
Total cost is the sum of the costs of the performed interventions. We would like to learn a causal graph with the minimum possible total cost. 

\noindent \textbf{Our Contributions:}
This is a natural problem that, to the best of our knowledge, has not been previously studied except for some special cases as we explain in the related work section. Our results are as follows: 
\begin{itemize} 
\item We show that the problem of designing the minimum cost interventions to learn a causal graph can be solved in polynomial time. 
\item We study the minimum cost intervention design problem when the number of interventions is limited. We formulate the cost-optimum intervention design problem as an integer linear program. This formulation allows us to identify two causal graph families for which the problem can be solved in polynomial time. 
\item For general graphs, we develop an efficient greedy algorithm. We also propose an improved variant of this algorithm, which runs in polynomial time when the causal graph is an interval graph. 
\end{itemize}

Our machinery is graph theoretic. We rely on the connection between graph separating systems and proper colorings. Although this connection was previously discovered, it does not seem to be widely known in the literature. 

\section{Background and Notation}
\label{sec:background}
In this section, we present a brief overview of Pearl's causality framework and illustrate how interventions are useful in identifying causal relations. We also present the requisite graph theory background. Finally, we explain separating systems: Separating systems are the central mathematical objects for non-adaptive intervention design.

\subsection{Causal Graphs, Interventions and Learning}
A causal graph  is a directed acyclic graph (DAG), where each vertex represents a random variable of the causal system. Consider a set of random variables $V$. A directed acyclic graph $D = (V,E)$ is a causal graph if the arrows in the edge set $E$ encode direct causal relations between the variables: A directed edge $X\rightarrow Y$ represents a direct causal relation between $X$ and $Y$. $X$ is said to be a direct cause of $Y$. In the structural causal modeling framework (also called structural equations with independent errors), every variable $X$ can be written as a deterministic function of its parent set in the causal graph $D$ and some unobserved random variable $E_X$. $E_X$ is called an exogenous variable and it is statistically independent from the non-descendants of $X$. Thus $X=f(Pa(X),E)$ where $Pa(X)$ is the set of the parents of $X$ in $D$ and $f$ is some deterministic function. We assume that the graph is acyclic\footnote{Treatment of cyclic graphs require mechanics different than independent exogenous variables, or a time varying system, and is out of the scope of this paper.} (DAG) and all the variables are observable (causal sufficiency). 

The functional relations between the observed variables and the exogenous variables induce a joint probability distribution over the observed variables. It can be shown that the underlying causal graph $D$ is a valid Bayesian network for the joint distribution induced over the observed variables by the causal model. To identify the causal graph, we can check the conditional independence relations between the observed variables. Under the faithfulness assumption \cite{Spirtes2001}, every conditional independence relation is equivalent to a graphical criterion called the \emph{d-separation} \footnote{The set of unfaithful distributions are shown to have measure 0. This makes faithfulness a widely employed assumption, even though it was recently shown that almost faithful distibutions may have significant measure \cite{Uhler2013}.}.

In general, there is no unique Bayesian network that corresponds to a given joint distribution: There exists multiple Bayesian networks for a given set of conditional independence relations. Thus, it is not possible to uniquely identify the underlying causal graph using only these tests in general. However, conditional independence tests allow us to identify a certain induced subgraph: \emph{Immoralities}, i.e., induced subgraphs on three nodes of the form $X\rightarrow Z\leftarrow Y$. An undirected graph $G$ is called the skeleton of a causal directed graph $D$, if every edge of $G$ corresponds to a directed edge of $D$, and every non-edge of $G$ corresponds to a non-edge of $D$. PC algorithm \cite{Spirtes2001} and its variants use conditional independence tests: They first identify the graph skeleton, and then determine all the immoralities. The runtime is polynomial if the underlying graph has constant vertex degree.

The set of invariant causal edges are not only those that belong to an immorality. For example, one can identify additional causal edges based on the fact that the graph is acyclic. Meek developed a complete set of rules in \cite{Meek1995a,Meek1995b} to identify every invariant edge direction, given a set of causal edges and the skeleton. Meek rules can be iteratively applied to the output of the PC algorithm to identify every invariant arrow. The graph that contains every invariant causal arrow as a directed edge, and the others as undirected edges is called the essential graph of $D$. Essential graphs are shown to contain undirected components which are always \emph{chordal} \footnote{A graph is chordal if its every cycle of length 4 or more contains a chord.}\cite{Spirtes2001,Hauser2012a} .

Performing experiments is the most definitive way to learn the causal direction between variables. Randomized clinical trials, which aim to measure the causal effect of a drug are examples of such experiments. In Pearl's causality framework, an experiment is captured through the \emph{do} operator: The \emph{do} operator refers to the process of assigning a particular value to a set of variables. An \emph{intervention} is an experiment where the scientist collects data after performing the \emph{do} operation on a subset of variables. This process is fundamentally different from conditioning, and requires scientist to have the power of changing the underlying causal system: For example, by forcing a patient not to smoke, the scientist removes the causal effect of the patient's urge to smoke which may be caused by a gene. An intervention is called perfect if it does not change any other mechanism of the causal system and only assigns the desired value to the intervened variable. A stochastic intervention assigns the value of the variable of interest to the realizations of another variable instead of a fixed value. The assigned variable is independent from the other variables in the system. This is represented as $do(X=U)$ for some independent random variable $U$. 

Due to the change of the causal mechanism, an intervention removes the causal arrows from $Pa(X)$ to $X$. This change in the graph skeleton can be detected by checking the conditional independences in the post-interventional distribution: The edges still adjacent to $X$ must have been directing away from $X$ before the experiment. The edges that are missing must have been the parents of $X$. Thus, an intervention on $X$ enables us to learn the direction of every edge adjacent to $X$. Similarly, intervening on a set of nodes $S\subseteq V$ concurrently enables us to learn the causal edges across the cut $(S,S^c)$.

Given sufficient data and computation power, we can apply the PC algorithm and Meek rules to identify the essential graph. To discover the rest of the graph we need to use interventions on the undirected components. We assume that we work on a single undirected component after this preprocessing step. Hence, the graphs we consider are chordal without loss of generality, since these components are shown to always be chordal \cite{Hauser2012a}. After each intervention, we also assume that the scientist can apply the PC algorithm and Meek rules to uncover more edges. A set of interventions is said to learn a causal graph given skeleton $G$, if every causal edge of any causal graph $D$ with skeleton $G$ can be identified through this procedure. A set of $m$ interventions is called an \emph{intervention design} and is shown by $\mathcal{I} = \{I_1,I_2,\hdots,I_m\}$, where $I_i\subset V$ is the set of nodes intervened on in the $i^{th}$ experiment. 

An intervention design algorithm is called non-adaptive if the choice of an intervention set does not depend on the outcome of the previous interventions. Yet, we can make use of the Meek rules over the hypothetical outcomes of each experiment. Adaptive algorithms design the next experiment based on the outcome of the previous interventions. Adaptive algorithms are in general hard to design and analyze and sometimes impractical when the scientist needs to design the interventions before the experiment starts. 

In this paper we are interested in the problem of learning a causal graph given its skeleton where each variable is associated with a cost. The objective is to non-adaptively design the set of interventions that minimizes the total interventional cost. We prove that, any set of interventions that can learn every causl graph with a given skeleton needs to be a graph separating system for the skeleton. This is the first formal proof of this statement to the best of our knowledge. 

\subsection{Separating systems, Graphs, Colorings}
A separating system on a set of elements is a collection of subsets with the following property: For every pair of elements from the set, there exists at least one subset which contains exactly one element from the pair:
\begin{definition} 
For set $V = [n]\coloneqq \{1,2,\hdots, n\}$, a collection of subsets of $V$, $\mathcal{I} = \{I_1,I_2,\hdots I_m\}$, is called a separating system if for every pair $u,v\in V$, $\exists i\in [m]$ such that either $u\in I_i$ and $v\notin I_i$, or $u\notin I_i$ and $v\in I_i$.
\end{definition}
The subset that contains exactly one element from the pair is said to separate the pair. The number of subsets in the separating system is called the size of the separating system. We can represent a separating system with a binary matrix:
\begin{definition}
Consider a separating system $\mathcal{I} = \{I_1,I_2,\hdots I_m\}$ for the set $[n]$. A binary matrix $\mat{M}\in\{0,1\}^{n\times m}$ is called the separating system matrix for $\mathcal{I}$ if for any element $j\in[n]$, $\mat{M}(j,i) = 1$ if $j\in I_i$ and 0 otherwise.
\end{definition}
Thus, each set element has a corresponding row coordinate, and the rows represent the set membership of these elements. Each column of $\mat{M}$ is a 0-1 vector that indicates which elements belong to the set corresponding to that column. See Figure \ref{fig:sepsys} for two examples. The definition of every pair being separated by some set then translates to every row of $M$ being different.

Given an undirected graph, a graph separating system is a separating system that separates every edge of the graph.
\begin{definition}
Given an undirected graph $G=([n],E)$, a set of subsets of $[n]$, $\mathcal{I} = \{I_1,I_2,\hdots I_m\}$, is a $G$-separating system if for every pair $u,v\in [n]$ for which $(u,v)\in E$, $\exists i\in [m]$ such that either $u\in I_i$ and $v\notin I_i$, or $u\notin I_i$ and $v\in I_i$.
\end{definition}
Thus, graph separating systems only need to separate pairs of elements adjacent in the graph. Graph separating systems are considered in \cite{Cheng1984}. It was shown that the size of the minimum graph separating system is $\lceil\log{\chi}\rceil$, where $\chi$ is the coloring number of $G$. Based on this, we can trivially extend the definition of separating system matrices to include graph separating systems.

A coloring of an undirected graph is an assignment of a set of labels (colors) to every vertex. A coloring is called proper if every adjacent vertex is assigned to a different color. A proper coloring for a graph is optimal if it is the proper coloring that uses the minimum number of colors. The number of colors used by an optimal coloring is the chromatic number of the graph. Optimum coloring is hard to find in general graphs, however it is in $P$ for perfect graphs. Since chordal graphs are perfect, the graphs we are interested in in this paper can be efficiently colored using minimum number of colors. For a given undirected graph $G=(V,E)$, the vertex induced subgraph on $S\subset V$ is shown by $G_S=\overline{(S,E)}$. 

\section{Related Work}
The framework of learing causal relations from data has been extensively studied under different assumptions on the causal model. \emph{The additive noise assumption} asserts that the effect of the exogenous variables are additive in the structural equations. Under the additional assumptions that the data is Gaussian and that the exogenous variables have equal variances, \cite{Peters2014} shows that the causal graph is identifiable. Recently, under the additive linear model with jointly Gaussian variables \cite{Peters2016} proposed using the invariance of the causal relations to combine a given set of interventional data.

For the case of two variable causal graphs, there is a rich set of theoretical results for data-driven learning:  \cite{Hoyer2008} and  \cite{Shimizu2006} show that we can learn a two-variable causal graph under different assumptions on the function or the noise term under the additive noise model. Alternatively, an information geometric appraoch that is based on the \emph{independence of cause and effect} is suggested by \cite{Janzing2012}. \cite{Paz2015a} recently proposed using a classifier on the datasets to label each dataset either as $X$ causes $Y$ or $Y$ causes $X$. The lack of large real causal datasets forced him to generate artificial causal data, which makes this approach dependent on the data generation process.  

Information theoretic causal inference approaches have gained increased attention recently \cite{Ziebart2013,Gao2016}. For time-series data along with Granger causality, directed information is used ~\cite{Granger1969,Etesami2016,Quinn2015,Kontoyiannis2016,Raginsky2011}. An entropic causal inference framework is recently proposed for the two-variable causal graphs by \cite{Kocaoglu2017}.

The literature on learning causal graphs using interventions without assumptions on the causal model is more limited. For the objective of minimizing the number of experiments, \cite{Hauser2012b} proposes a coloring-based algorithm to construct the optimum set of interventions. \cite{Eberhardt2005} introduced the constraint on the number of variables intervened in each experiment. He proved in \cite{EberhardtThesis} that, when all causal graphs are considered, the set of interventions to fully identify the causal DAG needs to be a separating system for the set of variables. For example for complete graphs, separating systems are necessary. \cite{Hyttinen2013}  draws connections between the combinatorics literature and causality via known separating system constructions. \cite{Shanmugam2015} illustrates several theoretical findings: They show that the separating systems are necessary even under the constraint that each intervention has size at most $k$, identify an information theoretic lower bound on the necessary number of experiments, propose a new $(n,k)$ separating system construction, and develop an adaptive algorithm that exploits the Meek rules. To the best of our knowledge, the fact that a graph separating system is necessary for a given causal graph skeleton was unknown until this work.

\section{Graph Separating Systems, Proper Colorings and Intervention Design}
In this section, we illustrate the relation between graph colorings and graph separating systems, and show how they are useful for non-adaptive intervention design algorithms.

Given a graph separating system $\mathcal{I} = \{I_1,I_2,\hdots,I_m\}$ for the skeleton $G$ of a causal graph, we can construct the set of interventions as follows: For experiment $i$, intervene on the set of variables in the set $I_i$. Since $\mathcal{I}$ is a graph separating system, for every edge in the skeleton, there is some $i$ for which $I_i$ intervenes on only one of the variables adjacent to that edge. Since the edge is cut, it can be learned by learning the skeleton of the post-interventional graph, as explained in Section \ref{sec:background}. Since every edge is cut at least once, an intervention design based on a $G$-separating system identifies any causal graph with skeleton $G$.

Graph separating systems provide a structured way of designing interventions that can learn any causal graph. Their necessity however is more subtle: One might suspect that using the Meek rules in between every intervention may eliminate the need for the set of interventions to correspond to a graph separating system. Suppose we designed the first $i-1$ experiments. Applying the Meek rules over all possible outcomes of our first $i-1$ experiments on $G$ may enable us to design the $m$th experiment in an informed manner, even though we do not get to see the outcome of our experiments. Eventually it might be possible to uncover the whole graph without having to separate every edge. In the following we show that Meek rules are not powerful enough to accomplish this, and we actually need a graph separating system. This fact seems to be known \cite{EberhardtThesis,Hauser2012b}, however we could not locate a proof. We provide our own proof:
\begin{theorem}
\label{thm:SepSysNecessary}
Consider an undirected graph $G$. A set of interventions $\mathcal{I}$ learns every causal graph $D$ with skeleton $G$ if and only if $\mathcal{I}$ is a graph separating system for $G$.
\end{theorem}
\begin{proof}
See the appendix.
\end{proof}

\subsection{Any Graph Separating System is \emph{Some} Coloring}
\label{sec:somecoloring}
In this section, we explain the relation between graph separating systems and proper graph colorings. This relation, which is already known \cite{Hauser2012b}, is important for us in reformulating the intervention design problem in the later sections.

Let $C:V\rightarrow \{0,1\}^m$ be a proper graph coloring for graph $G$ which uses $c$ colors in total. The colors are labeled by length-$m$ binary vectors. First construct matrix $\mat{M}$ as follows: Let $i^{th}$ row of $\mat{M}$ be the label corresponding to the color of vertex $i$, i.e., $C(i)$. Then $\mat{M}$ is a $G$-separating system matrix: Let $I_i$ be the set of row indices of $\mat{M}$ for which the corresponding entries in the $i^{th}$ column are 1. Let $\mathcal{I}=\{I_1,I_2,\hdots, I_m\}$ be the set of subsets constructed in this manner from $m$ columns of $\mat{M}$. Then $\mathcal{I}$ is a graph separating system for $G$. To see this, consider any pair of vertices $u, v$ that are adjacent in $G$: $(u,v)\in E$. Since the coloring is proper, the color labels of these vertices are different, which implies the corresponding rows of $\mat{M}$, $\mat{M}(u,:)$ and $\mat{M}(v,:)$ are different. Hence, there is some column of $\mat{M}$ which is 1 in exactly one of the $u^{th}$ and $v^{th}$ rows. Thus, the subset constructed from this column separates the pair of vertices $u,v$.

Therefore any proper graph coloring can be used to construct a graph separating system. It turns out that the converse is also true: Any graph separating system can be used to construct a proper graph coloring. This is shown by Cai in \cite{Cheng1984} within his proof that shows that the minimum size of a graph separating system is $\lceil\log{\chi}\rceil$, where $\chi$ is the chromatic number. We repeat this result for completeness\footnote{Note that this lemma is not formally stated in \cite{Cheng1984} but rather verbally argued within a proof of another statement.}:
\begin{lemma}[\cite{Cheng1984}]
\label{lem:cai}
Let $\mathcal{I}=\{I_1,I_2,\hdots, I_m\}$ be a graph separating system for the graph $G=(V,E)$. Let $\mat{M}$ be the separating system matrix for $\mathcal{I}$: $i^{th}$ column of $\mat{M}$ is the binary vector of length $|V|$ which is 1 in the rows that are contained in $I_i$. Then the coloring $C(i) = \mat{M}(i,:)$ is a proper coloring for $G$.
\end{lemma}

This connection between graph colorings and graph separating systems is important: Ultimately, we want to use graph colorings as a tool for searching over all sets of interventions, and find the one that minimizes a cost function. This is possible due to the characterization in Lemma \ref{lem:cai} and the fact that the set of interventions has to correspond to a graph separating system in order to identify any causal graph by Theorem \ref{thm:SepSysNecessary}. 

Along this direction, we have the following simple, yet important observation: We observe that a minimum graph separating system does not have to correspond to an optimum coloring. We illustrate this with a simple example:

\begin{figure*}[t]
\centering
\subfigure[An undirected graph]{\label{fig:graphsep}\includegraphics[width=60mm]{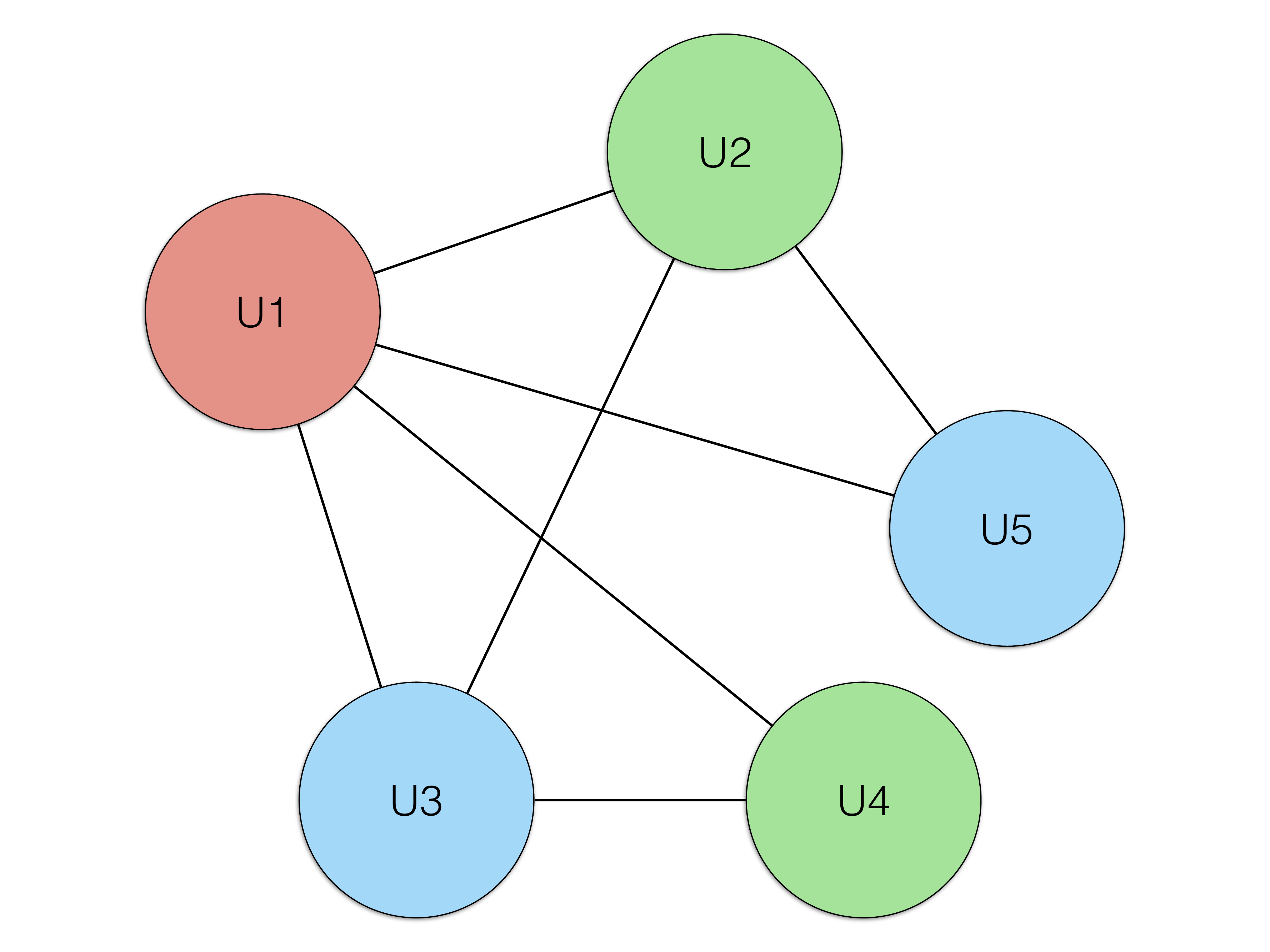}}
\hspace{1in}
\subfigure[Graph separating system vs. color separating system]{\label{fig:sepsys}\includegraphics[width=60mm]{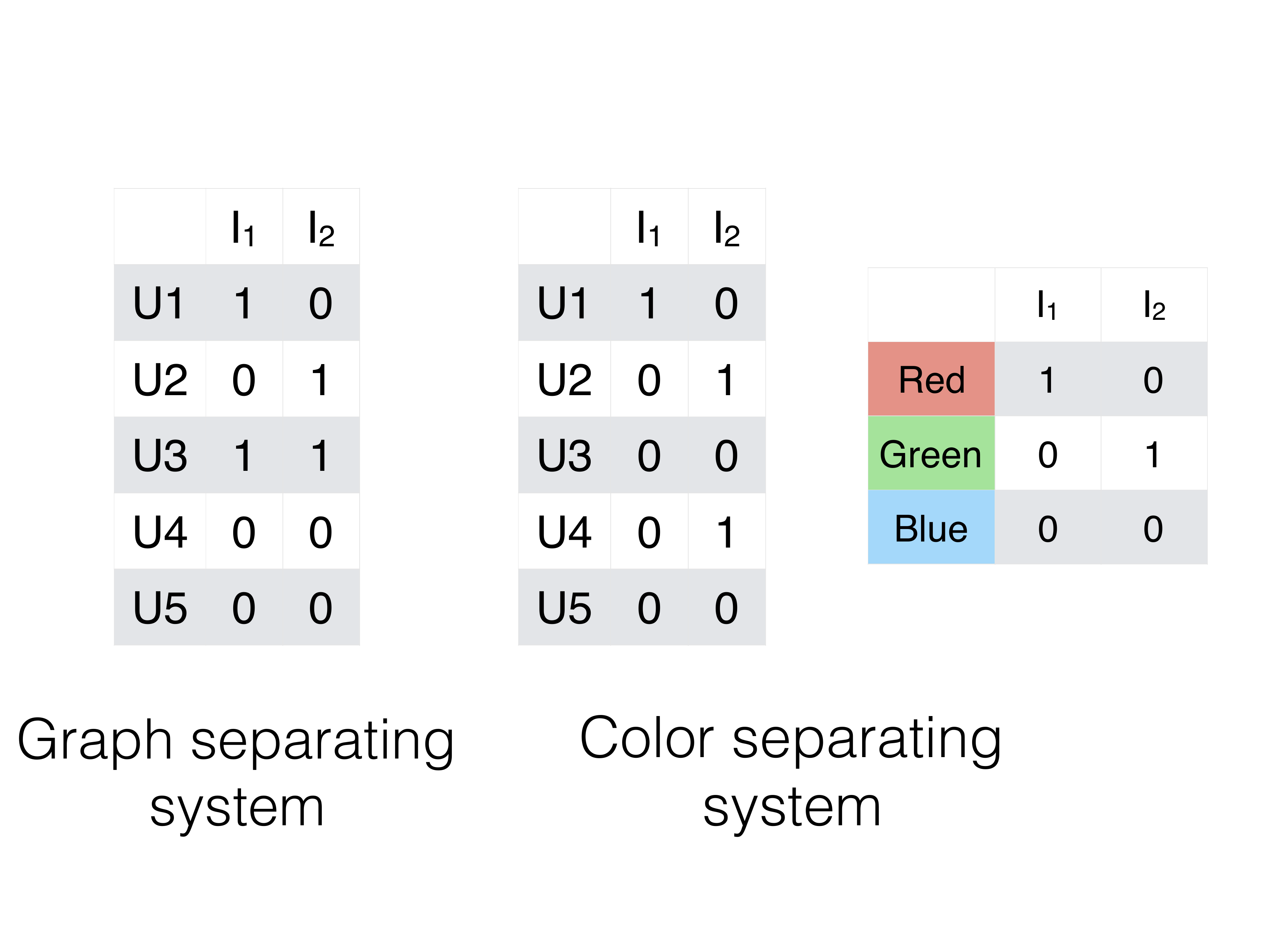}}
\caption{\small (a) An undirected graph with a proper 3 coloring. (b) A graph separating system, which does not separate color classes for any proper coloring of the graph. An example color-separating system is also provided.}
\end{figure*}

\begin{proposition}
Consider the undirected graph in Fig. \ref{fig:graphsep}. There does not exist any proper 3 coloring of this graph, for which the graph separating system given in Fig. \ref{fig:sepsys} separates every node across color classes.
\end{proposition}
\begin{proof}
Notice that the chromatic number of the given graph is 3. Hence the minimum separating system size is $\lceil\log_2(3)\rceil=2$. Thus the given graph separating system is a minimum graph separating system. In any proper 3-coloring, $U4$ and $U5$ must have different colors. Hence, any color-separating system separates $U4$ and $U5$. However the rows of the graph separating system which correspond to $U4$ and $U5$ are the same. In other words, any 3-coloring based graph separating system separates $U4$ and $U5$ whereas the graph separating system given in Fig. \ref{fig:graphsep} does not.
\end{proof}

This problem can be solved by assigning both vertices $U4$ and $U5$ a new color, hence coloring the graph by $\chi +1$ colors. We can conclude the following: Suppose we consider the cost-optimum intervention design problem with at most $\lceil\log(\chi)\rceil$ interventions. When we formulate it as a search problem over the graph colorings, we need to consider the colorings with at most $2^{\lceil\log(\chi)\rceil}$ colors instead of $\chi$ colors. 

\section{Cost-Optimal Intervention Design}
In this section, we first define the cost-optimal intervention design problem. Later we show that this problem can be solved in polynomial time. 

Suppose each variable has an associated cost $w_i$ of being intervened on. We consider a modular cost function: The cost of intervening on a set $S$ of nodes is $w(S) =\sum_{i\in S}w_i $. Our objective is to find the set of interventions with minimum total cost, that can identify any causal graph with the given skeleton: Given the causal graph skeleton $G$, find the set of interventions $\mathcal{S} = \{S_1,S_2,\hdots,S_m\}$ that can identify any causal graph with the skeleton $G$, with minimum total cost $\sum_{i}\sum_{j\in S_i}w_j$. In this section, we do not assume that the number of experiments are bounded and we are only interested in minimizing the total cost. We have the following theorem:
\begin{theorem}
\label{thm:unbounded}
Let $G=(V,E)$ be a chordal graph, and $w:V\rightarrow \mathbb{R}$ be a cost function on its vertices. Let an intervention on set $I$ have cost $\sum_{i\in I} w_i$. Then the optimal set of interventions with minimum total cost, that can learn any causal graph $D$ with skeleton $G$ is given by $\mathcal{I} = \{I_i\}_{i\in [\chi]}$, where $\{I_i\}_{i\in [\chi]}$ is any $\chi$ coloring of the graph $G_{V\backslash S}=\overline{(V\backslash S,E)}$, where $S$ is the maximum weighted independent set of $G$.
\end{theorem}
\begin{proof}
See the supplementary material.
\end{proof}

In other words, the optimum strategy is to color the vertex induced subgraph obtained by removing the maximum weighted independent set $S$ and intervening on each color class individually. After coloring the maximum weighted independent set, the remaining graph can always be colored by at most $\chi$ colors, i.e., the chromatic number of $G$. The remaining graph is still chordal. Since optimum coloring and maximum weighted independent set can be found in polynomial time for chordal graphs, $\mathcal{I}$ can be constructed in polynomial time.

\section{Intervention Design with Bounded Number of Interventions}
\label{sec:nonAdaptive}
In this section, we consider the cost-optimum intervention design problem for a given number of experiments. We construct a linear integer program formulation of this problem and identify the conditions under which it can be efficiently solved. As a corollary we show that when the causal graph skeleton is a tree or a clique tree, the cost-optimal intervention design problem can be solved in polynomial time. Later, we present two greedy algorithms for more general graph classes.

To be able to uniquely identify any causal graph, we need a graph separating system by Theorem \ref{thm:SepSysNecessary}. Hence, we need $m\geq \lceil\log(\chi)\rceil$ since the minimum graph separating system has size $\lceil\log(\chi)\rceil$ due to \cite{Cheng1984}. 

\subsection{Coloring formulation of Cost-Optimum Intervention Design}
One common approach to tackle combinatorial optimization problems is to write them as linear integer programs: Often binary variables are used with a linear objective function and a set of linear constraints. The constraints determine the set of feasible points. One can construct a convex object (a convex polytope) based on the set of feasible points by simply taking their convex hull. However this object can not always be described efficiently. If it can, then the linear program over this convex object can be efficiently solved and the result is the optimal solution of the original combinatorial optimization problem. We develop an integer linear program formulation for finding the cost-optimum intervention design using its connection to proper graph colorings. 

From Theorem \ref{thm:SepSysNecessary}, we know that we need the set of interventions to correspond to a graph separating system for the skeleton. From Lemma \ref{lem:cai}, we know that any graph separating system can be constructed from some proper coloring. Based on these, we have the following key observation: To solve the cost-optimal intervention design problem given a skeleton graph, \emph{it is sufficient to search over all proper colorings}, and find the coloring that gives the graph separating system with the minimum cost. We use the following (standard) coloring formulation: Suppose we are given an undirected graph $G$ with $n$ vertices and $t$ colors are available. Assign a binary variable $x_{i,k}\in\{0,1\}$ to every vertex-color pair $(i,k)$: $x_{i,k}=1$ if vertex $i$ is colored with color $k$, and 0 otherwise. Each vertex is assigned a single color, which can be captured by the equality $\sum_{k\in[t]}x_{i,k}=1$. Since coloring is proper, every pair of adjacent vertices are assigned different colors, which can be captured by $x_{i,k}+x_{j,k}\leq 1,\forall (i,j)\in E, \forall k\in [t]$. Based on our linear integer program formulation given in the supplementary material, we have the following theorem:
\begin{theorem}
\label{thm:search}
Consider the cost-optimal non-adaptive intervention design problem given the skeleton $G=(V,E)$ of the causal graph: Let each node be associated with an intervention cost, and the cost of intervening on a set of variables be the sum of the costs of each variable. Then, the non-adaptive intervention design that can learn any causal graph with the given skeleton in at most $m$ interventions with the minimum total cost can be identified in polynomial time, if the following polytope can be described using polynomially many linear inequalities:
\begin{align}
\label{eq:polytope}
\mathcal{C} = \conv \{x\in{\mathbb{R}^{n\times 2^m}}:&\sum\nolimits_{k\in [2^m]}x_{i,k}\leq 1, \forall i\in[n],\\
& x_{i,k}+x_{j,k}\leq 1, \forall (i,j)\in E,\nonumber\\
 & x_{i,k}\in\{0,1\},\forall i\in[n], k\in[2^m] \}\nonumber.
\end{align}
\end{theorem}
\begin{proof}
See the supplementary material.
\end{proof}
Donne in \cite{Donne2016} identifies that when the graph is a tree, one can replace the constraints $x_{i,k}\in \{0,1\}$ with $x_{i,k}\geq 0$ for all $(i,k)\in [n]\times [2^m]$ without changing the polytope in \ref{eq:polytope}. He also shows that when the graph is a clique-tree (a graph that can be obtained from a  tree by replacing the vertices of the tree with cliques), a simple alternative characterization based on the constraints on the maximum cliques of the graph exists, which can be efficiently described. Based on this and Theorem \ref{thm:search}, we have the following corollary:
\begin{corollary}
The cost-optimal non-adaptive intervention design problem can be solved in polynomial time if the given skeleton of the causal graph is a tree or a clique tree.
\end{corollary}

We can identify two other special solutions for the cost-optimum intervention design problem through a combinatorial argument: $(i)$ The maximum number of interventions is greater than or equal to the chromatic number $\chi$, $(ii)$ the graph is uniquely colorable. See the supplementary material for the corresponding results and details.

\subsection{Greedy algorithms}
In this section, we present two greedy algorithms for the minimum cost intervention design problem for more general graph classes. 
\begin{algorithm}[ht!]
\begin{small}
    \caption{Greedy Intervention Design for Total Cost Minimization for Chordal Skeleton}
   \label{alg:greedyChordal}
\begin{algorithmic}[1]
    \STATE {\bfseries Input:} A chordal graph $G$, maximum number of interventions $m$, cost $w_i$ assigned to each vertex $i$.
	\STATE $r=2^m, t=0, G_1=\overline{(V_1,E)}, V_1 = V$.
	\STATE $T=$ \rm{All binary vectors of length $m$}.
     \WHILE  {$r > \chi$} 
     \STATE Find maximum weighted independent set $S_t$ of $G_t$.
     \STATE Find $u  = \argmin_{x\in T} \lvert x\rvert_1$ (Break ties arbitrarily).
     \STATE Assign $M(i,:)=$ to every $i\in S_t$.
     \STATE $G_{t+1} = \overline{(V_{t+1},E)}, V_{t+1} = V_t\backslash S_t $: $G_{t+1}$ is the induced subgraph on the uncolored nodes.
 	\STATE $r\leftarrow r-1$, $t\leftarrow t+1 $, $T \leftarrow T-\{u\}$.
    \ENDWHILE
    \STATE Color $G_{t-1}$ with minimum number of colors.
    \STATE Assign the remaining length-$m$ binary vectors as rows of $M$ to different color classes.
    \STATE Output: $M$.
\end{algorithmic}
\end{small}
\end{algorithm}
\begin{algorithm}[ht!]
\begin{small}
    \caption{Greedy Intervention Design for Total Cost Minimization for Interval Skeleton}
   \label{alg:greedyInterval}
\begin{algorithmic}[1]
    \STATE {\bfseries Input:} An interval graph $G$, maximum number of interventions $m$, cost $w_i$ assigned to each vertex $i$.
	\STATE $r=2^m, t=0, G_1=\overline{(V_1,E)}, V_1 = V$.
     \WHILE  {$r-\binom{m}{t}\geq \chi$} 
     \STATE Find maximum (weighted) $\binom{m}{t}-$colorable induced subgraph $S_t$
     \STATE Assign all weight$-t$ binary vectors of length $m$ as rows of $M(S_t,:)$ to different color classes.
     \STATE $G_{t+1} = \overline{(V_{t+1},E)}, V_{t+1} = V_t\backslash S_t $: $G_{t+1}$ is the induced subgraph on the uncolored nodes.
     \STATE $r\leftarrow r-\binom{m}{t}$: $r$ is the number of unused available colors.
     \STATE $t\leftarrow t+1 $
    \ENDWHILE
    \STATE Color $G_{t-1}$ with minimum number of colors.
    \STATE Assign the remaining length-$m$ binary vectors as rows of $M$ to different color classes.
    \STATE Output: $M$.
\end{algorithmic}
\end{small}
\end{algorithm}

We have the following observation: Consider a coloring $C:V\rightarrow [t]$, which uses up to $t$ colors. Consider the graph separating system matrix $\mat{M}$ constructed using this coloring, as described in Section \ref{sec:somecoloring}. Recall that the $i^{th}$ row of $\mat{M}$ is a $\{0,1\}$ vector which represents the label for the color of vertex $i$, and $j^{th}$ column is the indicator vector for the set of variables included in intervention $j$. We call the $\{0,1\}$ vector used for color $k$ as the coloring label for color $k$. The separating property does not depend on the color labels: Using different labels for different colors is sufficient for the graph separating property to hold. However, the number of 1s of a coloring label determines how many times the corresponding variable is intervened on using the corresponding intervention design. Hence, we can choose the coloring labels from the binary vectors with small weight, given the choice. Moreover, the column index of a 1 in a certain row does not affect the cost since in a non-adaptive design, every intervention counts towards the total cost (we cannot stop the experiments earlier since we do not get feedback from the causal graph, unlike adaptive algorithms).

Based on this observation, we can try to greedily color the graph as follows: Suppose we are allowed to use up to $m$ interventions. Thus the corresponding graph separating system matrix $\mat{M}$ can have up to $m$ columns, which allows up to $2^m$ distinct coloring labels. We can greedily color the graph by choosing labels with small weight first: Choose the color label with smallest weight from the available labels. Find the maximum weighted independent set of the graph. Assign the coloring label to the rows associated with this the vertices in this independent set. Remove the used coloring label from the available labels, update the graph by removing the colored vertices and iterate. 

\begin{figure*}[ht!]
\centering
\subfigure[n=20]{\label{fig:Expn20}\includegraphics[width=0.3\textwidth]{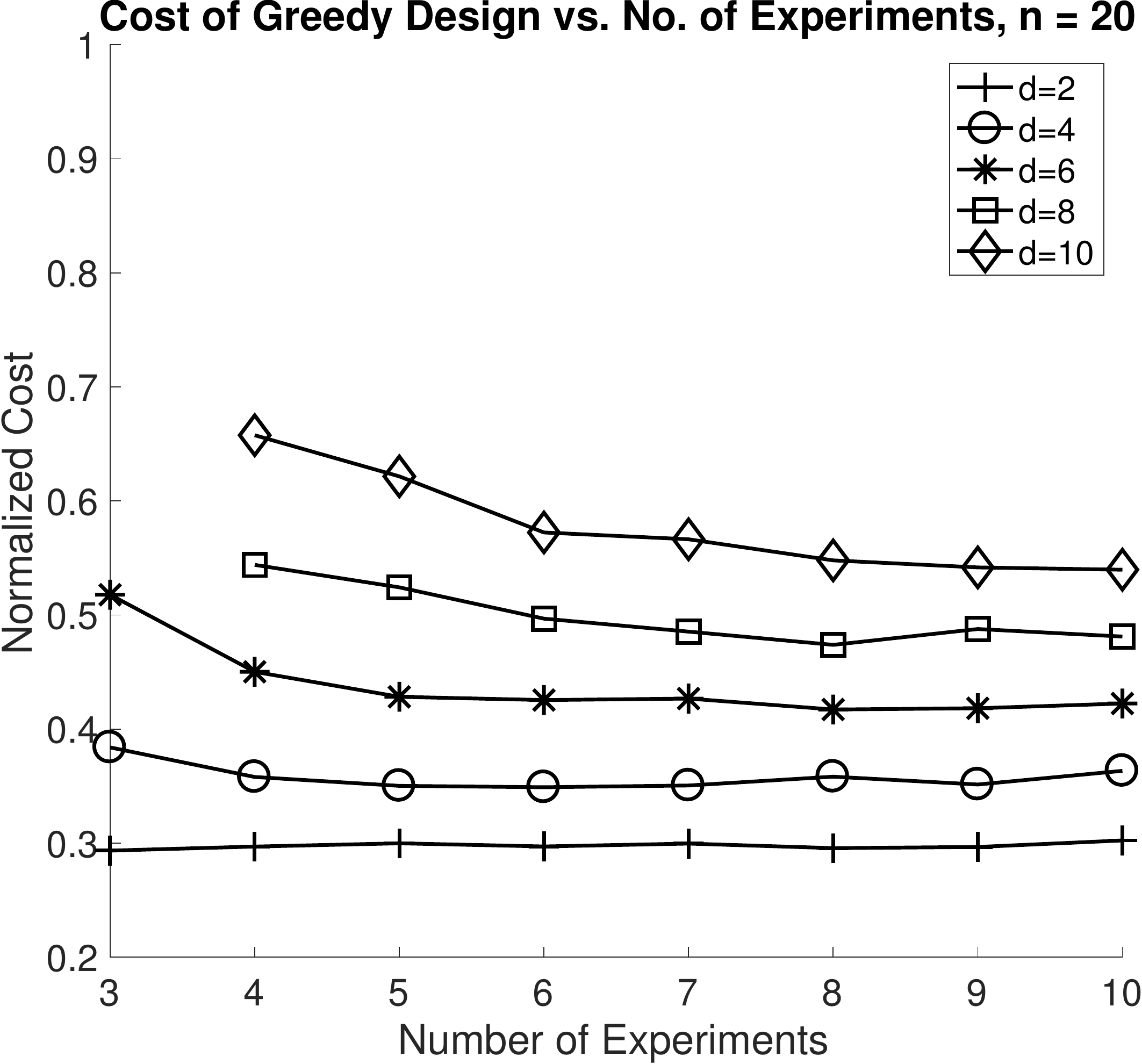}}
\subfigure[n=50]{\label{fig:Expn50}\includegraphics[width=0.3\textwidth]{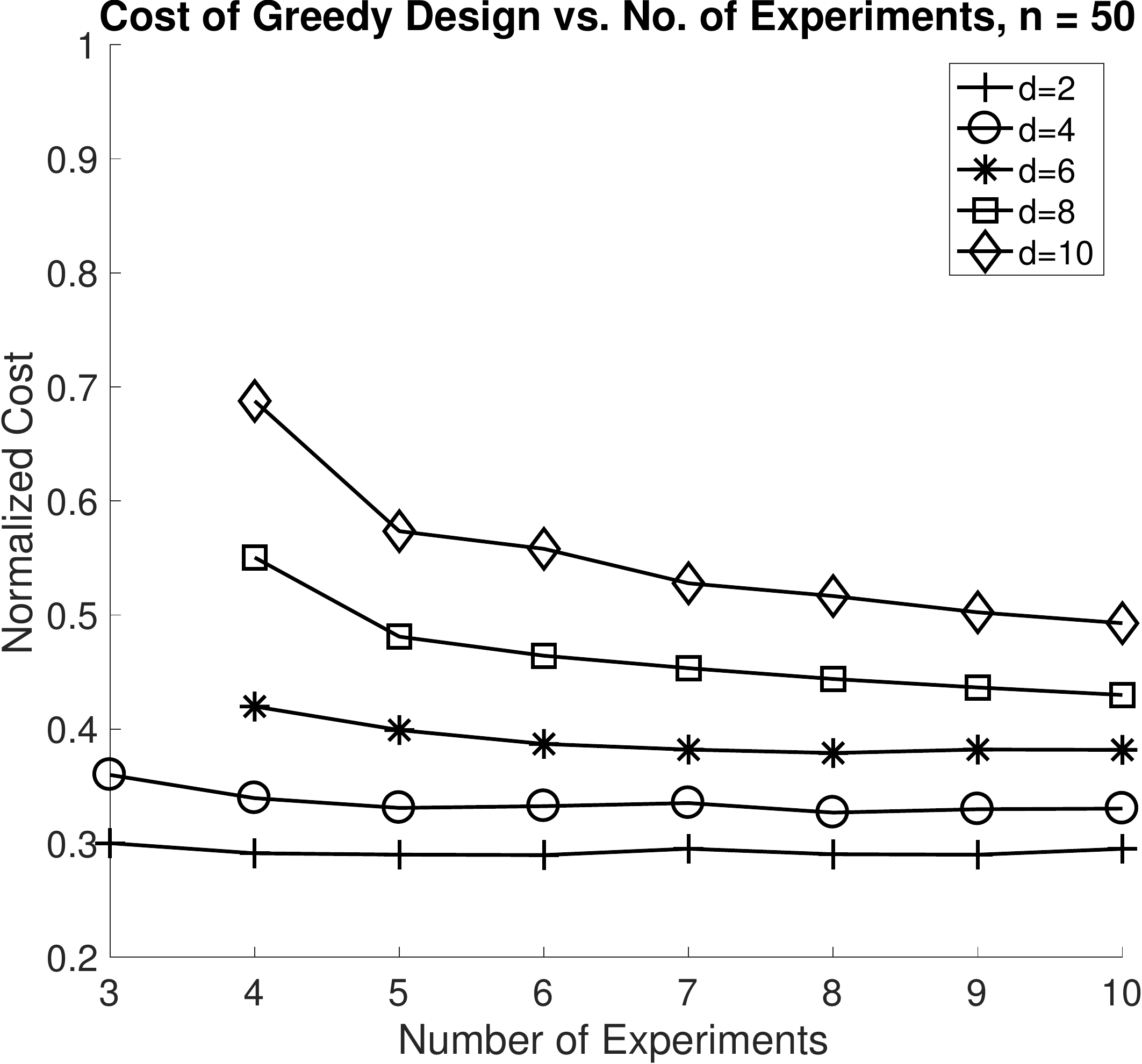}}
\subfigure[n=100]{\label{fig:Expn100}\includegraphics[width=0.3\textwidth]{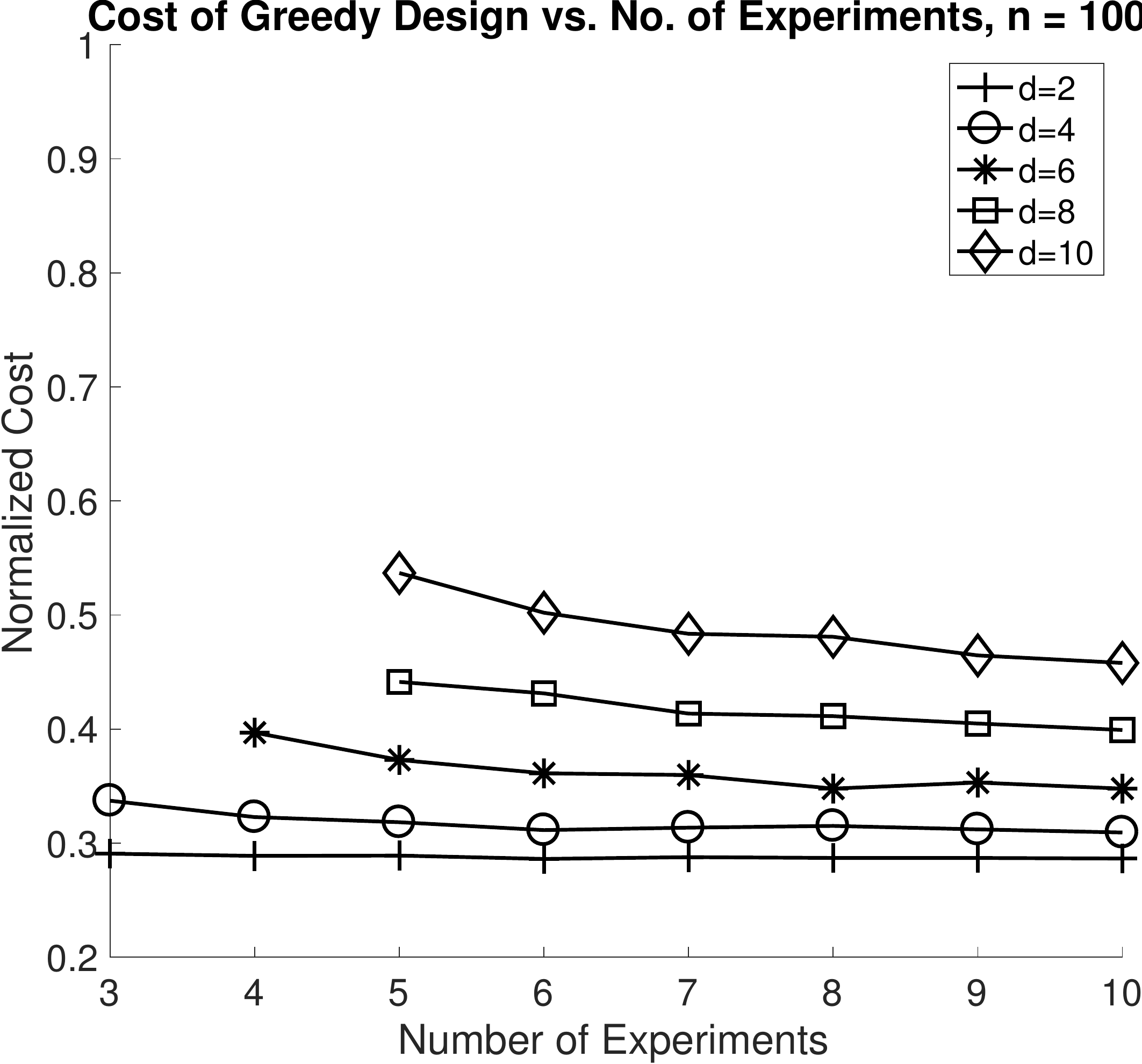}}
\caption{Exponential weights $w_i\sim exp(1)$. $n$: no. of vertices, $d$: Sparsity parameter of the chordal graph. Each datapoint is the average cost incurred by the greedy intervention design over 1000 randomly sampled causal graphs for a given number of experiments. The expected average cost of all the edges is $\E[w_i]=1$. The cost incurred by the intervention design is normalized by $n$. As observed, the cost incurred increases gradually as the number of experiments are reduced, or graph becomes denser. For sparse graphs, proposed construction incurs low cost even for up to 3 experiments.
}
\label{fig:simulation2}
\end{figure*}

However, this type of greedy coloring could end up using many more colors than allowed. Indeed one can show that greedily coloring a chordal graph using maximum independent sets at each step cannot approximate the chromatic number within an additive gap for all graphs. Thus, this vanilla greedy algorithm may use up all $2^m$ available colors and still have uncolored vertices, even though $\chi<2^m$. To avoid this, we use the following modified greedy algorithm: For the first $2^m-\chi$ steps, greedily color the graph using maximum weighted independent sets. Use the last $\chi$ colors to color the remaining uncolored vertices. Since the graph obtained by removing colored vertices have at most the same chromatic number as the original graph, $\chi$ colors are sufficient. The remaining graph is also chordal since removing vertices do not change the chordal property, hence finding a coloring that uses $\chi$ colors can be done efficiently. This algorithm is given in Algorithm \ref{alg:greedyChordal}.

We can improve our greedy algorithm when the graph is an interval graph, which is a strict subclass of the chordal graphs. Note that there are $\binom{m}{t}$ binary labels of length $m$ with weight $t$. When we use these $\binom{m}{t}$ vectors as the coloring labels, the corresponding intervention design requires every variable with these colors to be intervened on exactly $t$ times in total. Then, rather than finding the maximum independent set at iteration $t$, we can find the maximum weighted $\binom{m}{t}$-colorable subgraph, and use all the coloring labels of weight $t$. The cost of the colored vertices in the intervention design is $t$ times their total cost. We expect this to create a better coloring in terms of the total cost, since it colors a larger portion of the graph at each step. Finding the maximum weighted $k$ colorable subgraph is hard for non-constant $k$ in chordal graphs, however it can be solved in polynomial time if the graph is an interval graph \cite{Yannakakis1987}. This modified algorithm is given in Algorithm \ref{alg:greedyInterval}. Notice that when $m>>\log{n}$, the number of possible coloring labels is super-polynomial in $n$, which seem to make the algorithms run in super-polynomial time. However, when $m>>\log{n}$, we can only use the first $n$ color labels with the lowest weight, since a proper coloring on a graph with $n$ vertices can use at most $n$ colors in total.

\section{Experiments}
In this section, we test our greedy algorithm to construct intervention designs over randomly sampled chordal graphs. We follow the sampling scheme proposed by \cite{Shanmugam2015} (See the supplementary material for details). The costs of the vertices of the graph are selected from i.i.d. samples of an exponential random variable with mean 1. The total cost of all variables is then the same as the number of variables $n$. We normalize the cost incurred by our algorithm with $n$ and compare this normalized cost for different regimes. The parameter $d$ is a parameter that determines the sparsity of the graph: Graphs with larger $d$ are expected to have more edges.  See the supplementary material for the details of how the parameter $d$ affects the probability of an edge. We limit the simulation to at most 10 experiments ($x$-axis) and observe the effect of changing the number of variables $n$ and sparsity parameter $d$.  

Our intervention design algorithm, Algorithm \ref{alg:greedyChordal} requires a subroutine that can find the maximum weighted independent set of a given chordal graph. We implement the linear-time algorithm by Frank \cite{Frank1975} for finding the maximum weighted independent set of a chordal graph. For the details of Frank's implementation, see the supplementary material. 

We observe that the main factor that determines the average incurred cost is sparsity of the graph: The number of edges compared to the number of nodes. For a fixed $n$, reducing $d$ results in a smaller average cost by increasing the sparsity of the graph. For a fixed $d$, increasing $n$ reduces the sparsity, which is also shown to reduce the average cost incurred by the greedy intervention design. See the supplementary material for additional simulations where the costs are chosen as the i.i.d. sampes of a uniform random variable over the interval $[0,1]$. 
{\small
\bibliographystyle{plain}
\bibliography{causalinferenceSimple}
}
\newpage
\section{Appendix}

\subsection*{Proof of Theorem  \ref{thm:SepSysNecessary}}
One direction is trivial: Consider a $(G,\mathcal{C})$ separating system. For every edge there is an intervention where only one endpoint is intervened. This edge is in the cut and learned. Constraints are over the subsets of the graph separating system, which directly correspond to interventions. Hence interventions also obey the constraint $\mathcal{C}$. For the other direction, we use the following observation from \cite{Shanmugam2015}, which is implicit in the proof of Theroem 6 in \cite{Shanmugam2015}.
\begin{lemma}
\label{lem:shanmugam}
Let $G$ be an undirected chordal graph. Consider any clique $C$ of $G$. There is a directed graph $D$ with skeleton $G$ with no immoralities such that, the vertices $C$ come before any other vertices in the partial order that determine $D$. If this $D$ is the underlying causal graph, knowing the causal edges outside this clique does not help identify any edges within the clique.
\end{lemma}
The lemma essentially states that, Meek rules do not aid in identifying the edges within a clique, if the clique vertices come before any other vertex in the partial order of the underlying causal DAG. 

Assume that there is an edge that is not separated by the set of interventions. If the underlying causal DAG has partial order that starts with the nodes at the endpoints of this edge, then knowing every other edge does not help learn the direction of this edge by Lemma \ref{lem:shanmugam} (notice that an edge is a clique of size 2). Thus this set of interventions cannot learn every causal graph with the given skeleton.

\subsection*{Proof of Theorem \ref{thm:unbounded}}
Consider the graph separating system matrix $\mat{M}$: Let $\mat{M}\in\{0,1\}^{n,m}$ be a 0-1 matrix, where $\mat{M}(i,:)\neq \mat{M}(j,:), \forall (i,j)\in E$. Since every set of interventions must be a graph separating system by Theorem \ref{thm:SepSysNecessary}, we can work with the corresponding graph separating system matrices. Notice that any graph separating system corresponds to some proper coloring due to \ref{lem:cai}. Thus, any set of vertices that has identical rows in $\mat{M}$ should be within the same color class. We know in any proper coloring, each color class is an independent set. Then, over all proper colorings, the color class with maximum weight is given by the maximum weighted independent set. Since each row of $\mat{M}$ is either the all-zero vector, or contains at least a single 1, the total cost is minimized by assigning the all-zero vector to the vertices belonging to the maximum weighted independent set, and using distinct weight-1 vectors for the remaining rows. The induced graph on the vertices outside the maximum weighted independent set is still chordal and has the chromatic number at most $\chi$. Thus, we need an $n\times \chi$ matrix $\mat{M}$, hence $\chi$ experiments in total to minimize the total intervention cost. 

\subsection*{Proof of Theorem \ref{thm:search}}
In this section, we show that we can write the total cost of the interventions constructed by a given graph coloring can be written as a linear objective in terms of $x_{i,k}$. 

First, we illustrate the cost incurred by a given separating system. Consider the color separating system in Figure \ref{fig:sepsys}. Notice that the rows of $M$ that correspond to vertices within a fixed color class are the same. For example $S=\{U2,U4\}$ is a color class, and both rows are $[0, 1]$. Recall that the columns where a particular row is 1 indicate the interventions which contain that variable. The cost incurred by any vertex is the number of times the vertex is intervened on times the cost of intervening on that vertex. The cost incurred by a set of vertices is the sum of the cost incurred by each vertex within the set. Vertices within a color class are intervened on the same number of times since they have the same rows in the separating system matrix $\mat{M}$. Thus, the cost of a color class $S$ is given by $cost(S)=|r_S|_1\sum_{i\in S}w_i$, where $r_S$ is the row of any node from color class $S$ in $M$, and $|r_S|_1$ is the number of 1s in $r_S$. 

Notice that the exact labeling of rows do not matter for the separating system: We only need vertices with different colors to correspond to different rows. Since the cost of a color class is proportional to the number of 1s in its row vector, an optimum graph separating system given a coloring should assing vectors with smaller weight if possible, in order to minimize the total cost. Hence, in Figure \ref{fig:sepsys}, instead of assigning $[1,1]$ as the characteristic vector of $S$, we can assign $[0,1]$ without affecting the separating system property. Since 3 colors are sufficient, we do not need to use $[1,1]$ vector.


\begin{figure*}[ht!]
\vskip 0.2in
\begin{center}
\subfigure[n=20]{\label{fig:Unifn20}\includegraphics[width=0.3\textwidth]{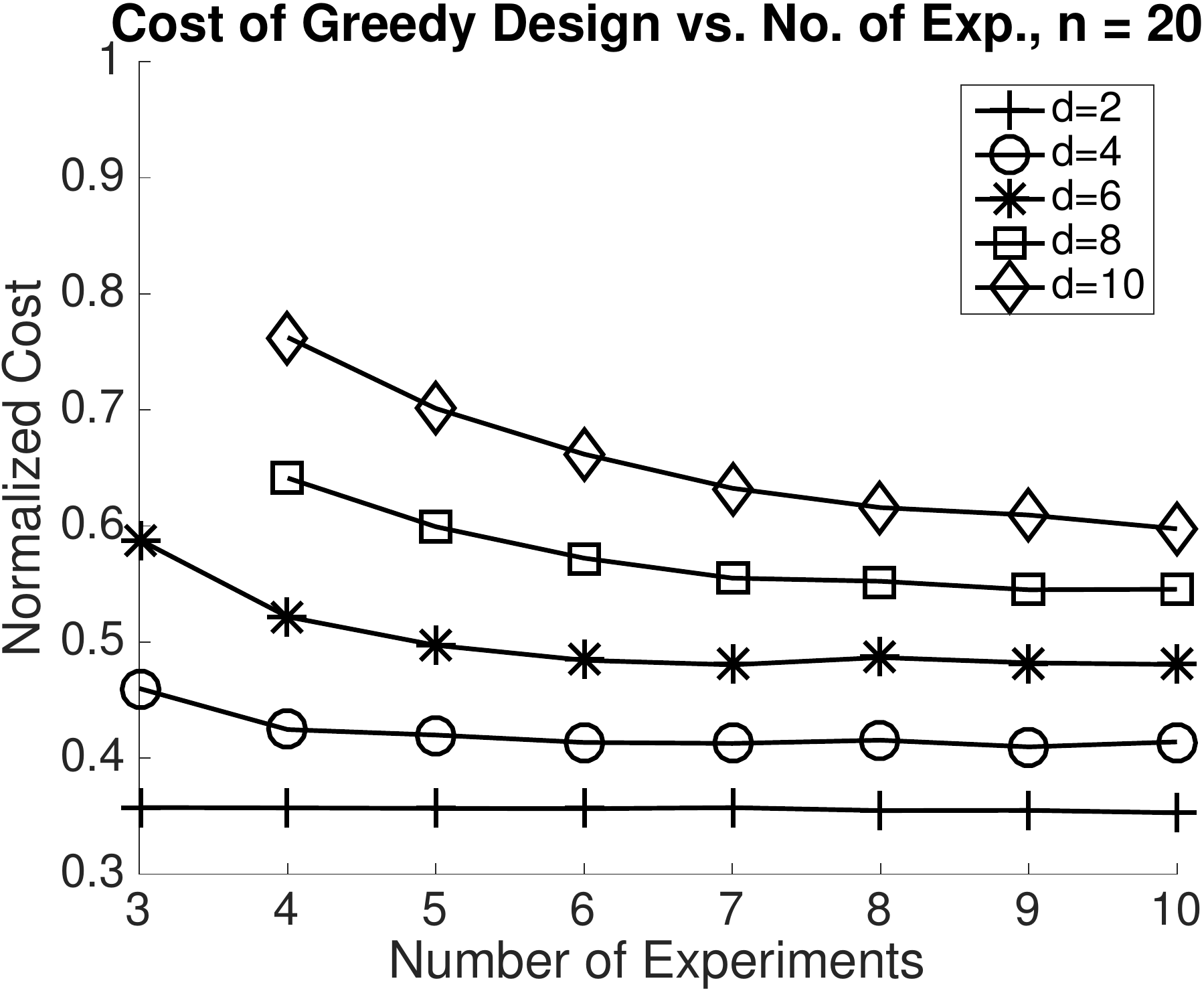}}
\subfigure[n=50]{\label{fig:Unifn50}\includegraphics[width=0.3\textwidth]{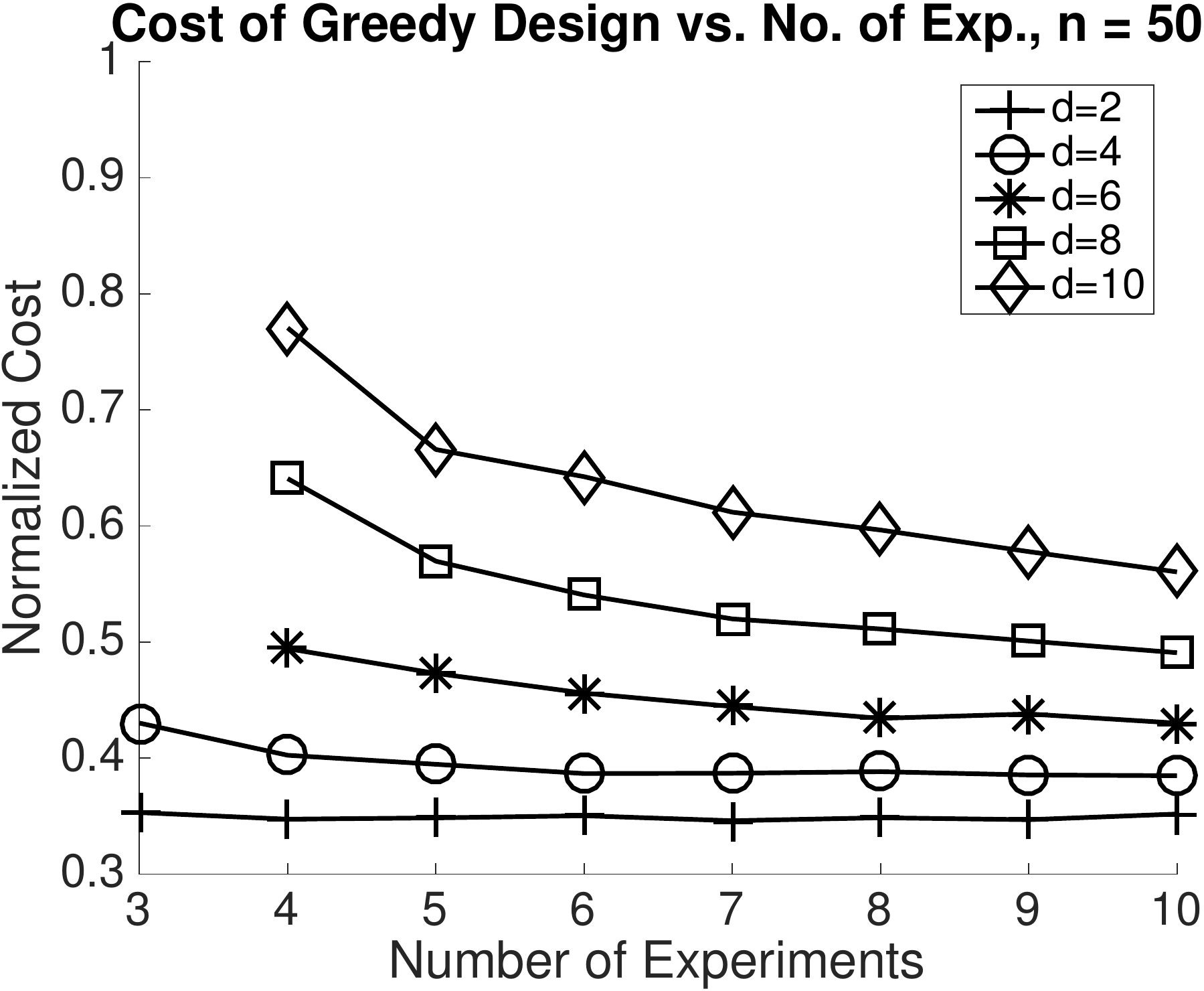}}
\subfigure[n=100]{\label{fig:Unifn100}\includegraphics[width=0.3\textwidth]{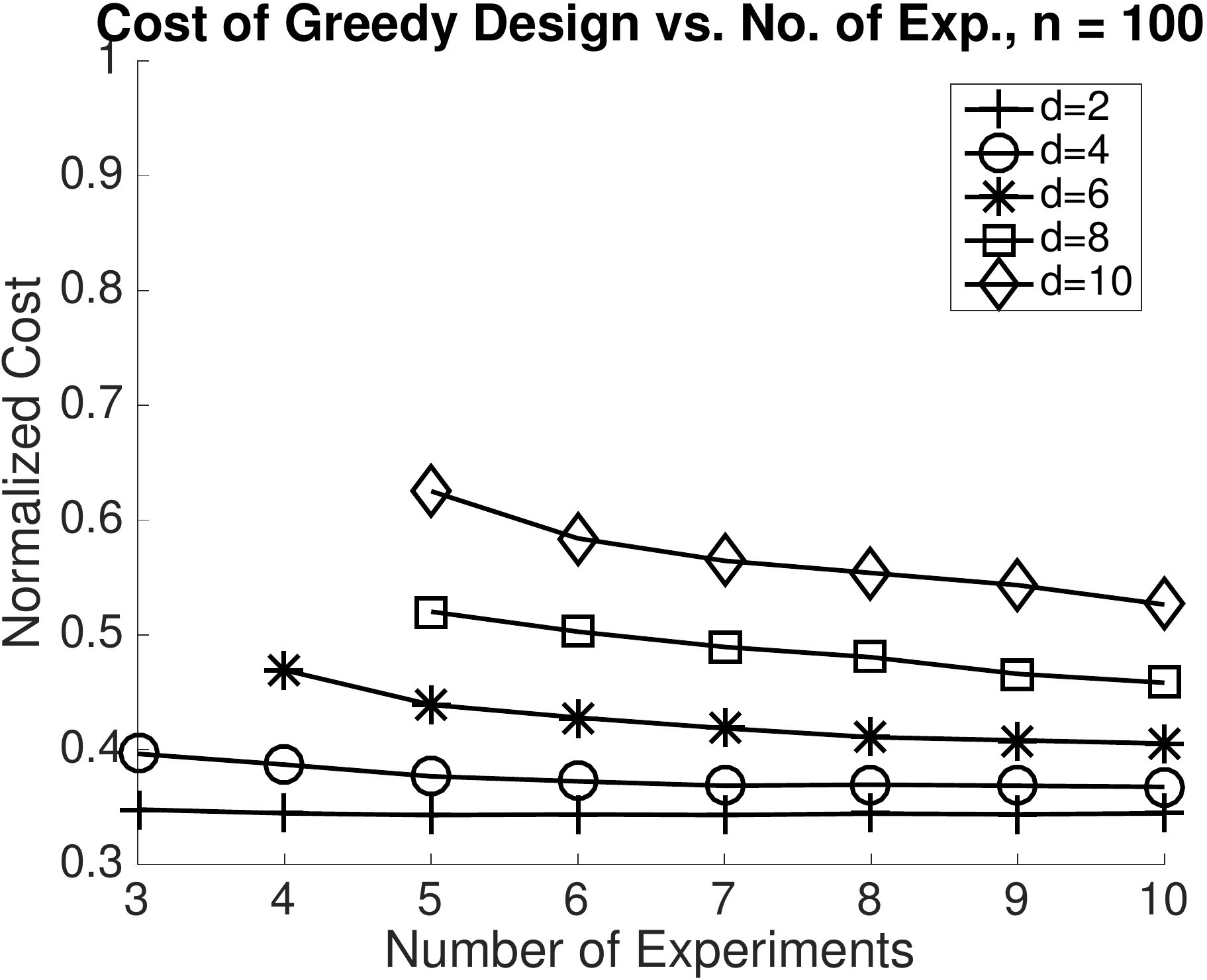}}
\caption{Uniform weights $w_i\sim \mathcal{U}[0,2]$. $n$: no. of vertices, $d$: Sparsity parameter of the chordal graph. Each datapoint is the average cost incurred by the greedy intervention design over 1000 randomly sampled causal graphs for a given number of experiments. The expected average cost of all the edges is $\E[w_i]=1$. The cost incurred by the intervention design is normalized by $n$. As observed, the cost incurred increases gradually as the number of experiments are reduced, or graph becomes denser. For sparse graphs, proposed construction incurs low cost even for up to 3 experiments.}
\label{fig:simulation1}
\end{center}
\vskip -0.2in
\end{figure*}

In general, given a number of interventions $m$, we need to construct a set of coloring labels to assign to each color. Suppose the causal graph has $n$ variables. If $m\leq\log{n}$, then every length-$m$ binary vector should be available, since the number of colors can be up to $n$. If $m>\log{n}$, using all labels give more number of colors than we can use to search over all proper colorings. Hence, in this case, we choose the labels with smallest weight until we find $n$ coloring labels. This ensures that the integer programming formulation does not have exponentially many variables, even when number of interventions is allowed to be $n$. Thus we construct a $\mat{b}$ vector, to be used as the weight of color labels as follows:

\begin{equation}
\mat{b} = [0, 1, 1, \hdots, 1, 2, 2, \hdots, 2,3, 3,\hdots, p, p, \hdots, p],
\end{equation}
where $p$ is such that $\sum_{i=0}^{p-1}\binom{m}{i}<\min{(2^m,n)}$ and $\sum_{i=0}^{p}\binom{m}{i}\geq \min{(2^m,n)}$. $i$ appears $\binom{m}{i}$ times if $i<p$ and $\min{2^m,n}-\sum_{j=0}^{p-1}\binom{m}{j}$ times if $i=p$. For notational convenience, let $t\coloneqq \min{(2^m,n)}$.

Standard coloring formulation assigns a variable $x_{i,j}$ to every node $i$ and color $j$: $x_{i,j}=1$ if node $i$ is colored with color $j$, and 0 otherwise. Each vertex is assigned a single color. Every pair of adjacent vertices are assigned different colors, which can be captured by $x_{i,k}+x_{j,k}\leq 1,\forall (i,j)\in E, \forall k\in [t]$. Then, using this standard coloring formulation, we can write our optimization problem as follows:
\begin{equation}
\begin{aligned}
&{\text{min}}  & & \sum_{j=0}^{t}\sum_{i=1}^n{w_ix_{i,j}\mat{b}(j)} \\
& \text{s. t.}
& &\sum_{j=1}^m x_{i,j} = 1,\forall i\in [n]\\
& & & x_{i,k}+x_{j,k}\leq 1 \forall (i,j)\in E,\forall k\in [t] \\
& & & x_{i,j}\in \{0,1\}\label{eq:mincostOptModular}
\end{aligned}
\end{equation}

\subsection*{Uniquely Colorable Graphs}
Next, we give a special case, which admits a simple solution without restricting the graph class. Suppose $G$ is uniquely $2^m-$ colorable, where $m$ is the maximum number of interventions we are allowed to use. Then there is only a single coloring up to permutations of colors. Hence the costs of color classes are fixed. Now we can simply sort the color classes in the order of decreasing cost, and assign row vectors of $\mat{M}$ to these color classes in the order of increasing number of 1s. This assures that the total cost of interventions is minimized.

\subsection*{Implementation Details}
First, we need to define a perfect elimination ordering:
\begin{definition}
A perfect elimination ordering (PEO) $\sigma_p=\{v_1,v_2 \ldots v_n\}$ on the vertices of an undirected chordal graph $G$ is such that for all $i$, the induced neighborhood of $v_i$ on the subgraph formed by $\{v_1,v_2 \ldots v_{i-1} \}$ is a clique.
\end{definition}

It is known that an undirected graph is chordal if and only if it has a perfect elimination ordering. We use this fact to generate chordal graphs based on a randomly chosen perfect elimination ordering: First we choose a random permutation to be the perfect elimination ordering for the chordal graph. Then the $i^{th}$ vertex is connected to each node in $S_i = \{j: j<i \text{ with respect to PEO}\}$ with respect to the PEO independently randomly with probability $\left(\frac{d}{i}\right)^{2/3}$. A random vertex from $S_i$ is chosen to be a parent of $i$ with probability 1 to keep the graph connected. The parent set are connected to each other to assure the ordering is a PEO.

\subsection*{Frank's Algorithm}
Consider a PEO $\sigma=\{v_n,v_{n-1},\hdots,v_1\}$. At step $i$, skip the vertex $v_i$ if it has weight $w_i=0$. Otherwise, mark it red and reduce the weight of all its neighbors that are before $v_i$ in the PEO by $w_i$, and set $w_i=0$. After $n$ steps, we have a set of vertices colored red. Parse this set in the order of $\sigma$ and convert a red vertex to blue if it does not have any neighbor $j<i$ in $\sigma$ which is already colored blue. \cite{Frank1975} proves that this algorithm outputs the maximum weighted independent set. 

\subsection*{Additional Simulations}
In this section we provide additional simulations for when the graph weights are uniformly distributed $w_i\sim \mathcal{U}[0,2]$. The results are given in Figure \ref{fig:simulation1}. Similar to the exponentially distributed weigths, the main factor determining the cost is the graph spartiy, which is captured by parameter $d$.

\end{document}